\documentclass[11pt,letterpaper]{article}

\usepackage{amsmath,amsthm,nicefrac}
\usepackage{amsfonts, amstext}

\usepackage{typearea}
\typearea{14}
\usepackage[font={small,it}]{caption}

\usepackage{setspace}

\usepackage{enumitem}
\usepackage[breaklinks]{hyperref}
\hypersetup{colorlinks=true,%
            citebordercolor={.6 .6 .6},linkbordercolor={.6 .6 .6},%
citecolor=blue,urlcolor=black,linkcolor=blue}

\usepackage[nameinlink]{cleveref}
\Crefname{algocf}{Algorithm}{Algorithms}
\crefname{algocfline}{line}{lines}
\Crefname{invariant}{Invariant}{Invariants}
\Crefname{claim}{Claim}{Claims}
\Crefname{subclaim}{Subclaim}{Subclaims}

\usepackage{epsfig}
\usepackage{amsthm,amssymb}
\usepackage{amsthm,amssymb}

\usepackage{mathrsfs}
\usepackage{xspace}
\usepackage{soul}
\usepackage{latexsym}
\usepackage{bbm}

\usepackage{framed}

\usepackage[dvipsnames]{xcolor}
\definecolor{DarkGray}{rgb}{0.66, 0.66, 0.66}
\definecolor{DarkPowderBlue}{rgb}{0.0, 0.2, 0.6}
\definecolor{fluorescentyellow}{rgb}{0.8, 1.0, 0.0}

\usepackage[ruled,vlined,linesnumbered,algonl]{algorithm2e}
\SetEndCharOfAlgoLine{}
\SetKwComment{Comment}{\footnotesize$\triangleright$\ }{}

\SetCommentSty{mycommfont}

\usepackage{thmtools,thm-restate}

\allowdisplaybreaks

\newcommand{\alert}[1]{{\color{red}#1}}

\newcounter{note}[section]

\sethlcolor{fluorescentyellow}

\newcommand{\initOneLiners}{%
    \setlength{\itemsep}{0pt}
    \setlength{\parsep }{0pt}
    \setlength{\topsep }{0pt}
}

\newtheorem{theorem}{Theorem}[section]
\newtheorem{lemma}[theorem]{Lemma}

\theoremstyle{definition}

\theoremstyle{remark}

\makeatletter 
\renewcommand{\theinvariant}{(I\@arabic\c@invariant)}
\makeatother

\newcommand{\fac}{{\textsf{Facility Location}}\xspace}

\newcommand{\opt}{{\textsf{opt}}}

\newcommand{\junk}[1]{}
\newcommand{\eat}[1]{}

\newif\ifhideproofs

\ifhideproofs
\usepackage{environ}
\NewEnviron{hide}{}

\fi

\newcommand{\ocp}{{\sc ocp}\xspace}

\newcommand{\ml}{{\sc ml}\xspace}
\newcommand{\static}{{\textsc{static}}\xspace}

\newcommand{\dynamic}{{\textsc{dynamic}}\xspace}
\newcommand{\xd}{x^{\textsc{dyn}}\xspace}
\newcommand{\yd}{y^{\textsc{dyn}}\xspace}

\begin{document}
\title{Online Algorithms with Multiple Predictions}

\author{
{Keerti Anand\thanks{Department of Computer Science, Duke University, Durham, NC. Email: {\tt kanand@cs.duke.edu}.}}
\and
{Rong Ge\thanks{Department of Computer Science, Duke University, Durham, NC. Email: {\tt rongge@cs.duke.edu}.}}
\and
{Amit Kumar\thanks{Department of Computer Science and Engineering, IIT Delhi, New Delhi, India. Email: {\tt amitk@cse.iitd.ac.in}.}}
\and
{Debmalya Panigrahi\thanks{Department of Computer Science, Duke University, Durham, NC. Email: {\tt debmalya@cs.duke.edu}.}}
}

\maketitle

\begin{abstract}
    This paper studies online algorithms augmented with {\em multiple} machine-learned predictions. While online algorithms augmented with a single prediction have been extensively studied in recent years, the literature for the multiple predictions setting is sparse. In this paper, we give a generic algorithmic framework for online covering problems with multiple predictions that obtains an online solution that is competitive against the performance of the {\em best} predictor. Our algorithm incorporates the use of predictions in the classic potential-based analysis of online algorithms. We apply our algorithmic framework to solve classical problems such as online set cover, (weighted) caching, and online facility location in the multiple predictions setting. Our algorithm can also be {\em robustified}, i.e., the algorithm can be simultaneously made competitive against the best prediction and the performance of the best online algorithm (without prediction).
    
\eat{   
    Online algorithms have adversarial assumptions about the future that lead to pessimistic albeit robust theoretical guarantees on performance. A recent trend in moving beyond worst-case algorithm design has been to incorporate predictions about the future. In this work, we work in the framework of multiple predictions that arrive as suggestions. Specifically, we concern ourselves with online covering where the predictions are in the form of an advice to satisfy the constraints. We define a natural benchmark for the meta-algorithm: an offline algorithm that follows the best suggestion at every time-step, and give provable guarantees on its performance with respect to this benchmark, whilst maintaining traditional robustness results. We also show matching lower bounds that imply that better performance can not be achieved in the general case. Furthermore, we extend the application of this framework to numerous use-cases such as Online Set Cover, Online Caching and Online Facility Location.
}
\end{abstract}

\thispagestyle{empty}

\clearpage
\setcounter{page}{1}

\section{Introduction}\label{sec:introduction}In many real world computational tasks, parts of the input are not known in advance and are revealed piecemeal over time. However, the algorithm is constrained to take decisions before the entire input is revealed, thereby optimizing for an unknown future. For instance, when an online retailer has to decide the locations of warehouses to serve clients, it does so without precisely knowing how the clientele will grow over time. Similarly, in an operating system, the cache scheduler has to decide which pages to evict from the cache without knowing future requests for page access. These kinds of scenarios are traditionally captured by the field of {\em online algorithms}, where the algorithm makes irrevocable decisions without knowing the future. The performance of an online algorithm is measured by its {\em competitive ratio} which is defined as the maximum ratio across all inputs between the cost of the online algorithm and that of an optimal solution (see, e.g., \cite{BorodinE98}). While this is a robust guarantee that holds for {\em all} inputs, the robustness comes at the cost of making online algorithms excessively cautious thereby resulting in strong lower bounds and also affecting their real world performance.

To overcome the pessimistic behavior of online algorithms, there has been a growing trend in recent years to incorporate machine-learned predictions about the future. This exploits the fact that in many real world settings, modern ML methods can predict future behavior to a high degree of accuracy. Formalized by Lykouris and Vassilvitksii~\cite{LykourisV18,LykourisV21} for the caching problem, the online algorithms with prediction framework allows online algorithms to access predicted future input values, but does not give any guarantee on the accuracy of such predictions. (This reflects the fact that ML predictions, say generated by a neural network, are usually without worst-case guarantees, and can occasionally be completely wrong.) The goal is to design online algorithms whose competitive ratio gracefully interpolates between offline algorithms if the predictions are accurate -- a property called {\em consistency} -- and online algorithms irrespective of predictions -- a property called {\em robustness} (these terms were coined by Kumar, Purohit, and Svitkina~\cite{KumarPS18}).

Online algorithms with predictions have been extensively studied in the last few years for a broad range of problems such as variants of ski rental~\cite{KumarPS18,KKP13,GollapudiKP19,WeiZ20,AnandGP20,WangLW20}, set cover~\cite{BamasMS20}, scheduling~\cite{KumarPS18,WeiZ20,BamasMRS20,LattanziLMV20,Mitzenmacher20,LeeMHLSL21,AzarLT21}, caching~\cite{LykourisV18,Wei20,JiangPS20,BansalCKPV20}, matching and secretary problems~\cite{lavastida2020learnable,DuttingLLV21,AntoniadisGKK20,JiangLT021}, metric optimization~\cite{AntoniadisCEPS20,AzarPT22,FotakisGGP21,JiangLLTZ21,AlmanzaCLPR21}, data structures~\cite{Mitzenmacher18}, statistical  estimation~\cite{HsuIKV19,IndykVY19,EdenINRSW21}, online search~\cite{AnandGKP21}, and so on. 
\eat{
Another related body of work is in data driven algorithm design that aims to select the best algorithmic configuration for a given problem instance (see survey~\cite{balcan2020data}). This was first introduced using a PAC-learning framework by\cite{gupta2017pac}. \cite{balcan2018dispersion} explored online algorithm selection. Other relevant problems that have been studied include, Integer Quadratic Programming~\cite{balcan2017learning}, greedy heuristics for knapsack~\cite{balcan2020semi}, clustering~\cite{garg2018supervising}, simulated annealing~\cite{blum2021learning},  portfolio selection~\cite{BalcanSV21} etc.}

In this paper, we focus on online algorithms with {\em multiple} machine-learned predictions. In many situations, different ML models and techniques end up with distinct predictions about the future, and the online algorithm has to decide which prediction to use among them. Indeed, this is also true of human experts providing inputs about expectations of the future, or other statistical tools for predictions such as surveys, polls, etc. Online algorithms with multiple predictions were introduced by Gollapudi and Panigrahi~\cite{GollapudiKP19} for the ski rental problem, and has since been studied for multi-shop ski rental~\cite{WangLW20} and facility location~\cite{AlmanzaCLPR21}. Furthermore, \cite{bhaskara2020online} considers multiple hints for regret minimization in Online Linear Optimization. In our current paper, instead of focusing on a single problem, we extend the powerful paradigm of {\em online covering problems} to incorporate multiple predictions. As a consequence, we obtain online algorithms with multiple predictions for a broad range of classical problems such as {\em set cover}, {\em caching}, and {\em facility location} as corollaries of the general technique that we develop in this paper.

\paragraph{The Online Covering Framework.}
Online covering is a powerful framework for capturing a broad range of problems in combinatorial optimization. In each online step, a new linear constraint $\mathbf{a}\cdot \mathbf{x}\ge b$ is presented to the algorithm, where $\mathbf{x}$ is the vector of variables, $\mathbf{a}$ is a vector of non-negative coefficients, and $b$ is a scalar. The algorithm needs to satisfy the new constraint, and is only allowed to increase the values of the variables to do so. The goal is to minimize an objective function $\mathbf{c} \cdot \mathbf{x}$, where $\mathbf{c}$ is the vector of costs that is known offline. This formulation captures a broad variety of problems including set cover, (weighted) caching, revenue maximization, network design, ski rental, TCP acknowledgment, etc. Alon~{\em et al.}~\cite{AlonAABN09} proposed a multiplicative weights update (MWU) technique for this problem and used it to solve the online set cover problem. This was quickly adapted to other online covering problems including weighted caching~\cite{BansalBN07}, network design~\cite{AlonAABN06}, allocation problems for revenue maximization~\cite{BuchbinderJN07}, etc. (The reader is referred to the survey~\cite{BuchbinderN09b} for more examples.) All these algorithms share a generic method for obtaining a fractional solution to the online covering problem, which was formalized by Buchbinder and Naor~\cite{BuchbinderN09a} and later further refined by Gupta and Nagarajan~\cite{GuptaN14}. Since then, the online covering problem has been generalized to many settings such as convex (non-linear) objectives~\cite{AzarBCCCGHKNNP16} and mixed covering and packing problems~\cite{AzarBFP13}. 

\paragraph{Comparison with Prior Work on Online Covering with ML Prediction.}
Bamas, Maggiori, and Svensson~\cite{BamasMS20} were the first to consider the online covering framework in the context of ML predictions. In a beautiful work, they gave the first general-purpose tool for online algorithms with predictions, and showed that this can be used to solve several classical problems like set cover and dynamic TCP acknowledgment. In their setting, a solution is presented as advice to the online algorithm at the outset, and the algorithm incorporates this suggestion in its online decision making. %Specifically, their approach was to extend the classical online primal dual paradigm by using the suggestion to control updates to the primal solution.

In our current paper, we give a general scheme for the online covering framework with {\em multiple} predictions. In particular:
\begin{enumerate}
\item [--] Since we are in the multiple predictions setting, we allow $k > 1$ suggestions instead of just a single suggestion, and benchmark our algorithm's performance against the best suggested solution. (Of course, the best suggestion is not known to the algorithm.)
\item[--] In contrast to Bamas {\em et al.}~\cite{BamasMS20}, we do not make the assumption that the entire suggested solution is given up front. Instead, in each online step, each of the $k$ suggestions gives a feasible way of satisfying the new constraint. Note that this is more general than giving the suggested solution(s) up front, since the entire solution(s) can be presented in each online step as a feasible way of satisfying the new constraint.
\item[--] In terms of the analysis, we extend the potential method from online algorithms (in contrast, Bamas~{\em et al.}~\cite{BamasMS20} use the primal dual framework). The potential method has been used recently for many classic problems in online algorithms such as weighted paging~\cite{BansalBN-pot}, $k$-server~\cite{BuchbinderGMN19}, metric allocation~\cite{BansalC21}, online set cover~\cite{BuchbinderGMN19}, etc. In fact, it can also be used to reprove the main results of Bamas~{\em et al.}~\cite{BamasMS20} in the single prediction setting. In this paper, we extend this powerful tool to incorporate multiple ML predictions.  %It is possible that the primal dual framework can also be extended to our multiple predictions setting, but the flexibility of the potential method helps us obtain a clean analysis of our algorithm.
\item[--] Finally, we show that our techniques extend to a generalization of the online covering framework to include {\em box}-type constraints. This extension allows the framework to handle more problems such as online facility location that are not directly captured by the online covering framework.
\end{enumerate}

\paragraph{Comparison with Online Learning.}
The reader will notice the similarity of our problem to the classical experts' framework from online learning (see, e.g., the survey~\cite{Shalev-Shwartz12}). In the experts' framework, each of $k$ experts provides a suggestion in each online step, and the algorithm has to choose (play) one of these $k$ options. After the algorithm has made its choice, the cost (loss) of each suggestion is revealed before the next online step. The goal of the algorithm is to be competitive with the {\em best expert} in terms of total loss. In contrast, 
\begin{enumerate}
    \item [--] Since we are solving a combinatorial problem, the (incremental) cost of any given step for an expert or the algorithm depends on their pre-existing solution from previous steps (therefore, in particular, even after following an expert's choice, the algorithm might suffer a larger incremental cost than the expert). This is unlike online learning where the cost in a particular step is independent of previous choices.
    \item[--] In online learning, the algorithm is benchmarked against the best {\em static} expert in hindsight, i.e., the best solution whose choices match that of the same expert across all the steps. Indeed, it can be easily shown that no algorithm can be competitive against a {\em dynamic} expert, namely a solution that chooses the best suggestion in each online step even if those choices come from different experts. Observe that such a dynamic expert can in general perform much better than each of the suggestions, e.g., when the suggestions differ from each other but at each time, at least one of them suggests a good solution. But, in our problem, since the choices made by experts correspond to solutions of a combinatorial problem, we can actually show that our algorithm is competitive even against a dynamic expert. Namely, the $k$ suggestions in every step are not indexed by specific experts, and the algorithm is competitive against any composite solution that chooses any one of the $k$ suggestions in each step. 
    %
    %Interestingly, it actually turns out that there is no significant difference for our problem between choosing a static or a dynamic expert in terms of results. In fact, we show that our lower bounds holds against static experts (and therefore against the more powerful dynamic experts), while the competitive ratio of our algorithm holds against the best dynamic expert (and therefore against the weaker static experts as well).
    \item[--] In online learning, the goal is to obtain {\em regret} bounds that show that the online algorithm approaches the best (static) expert's performance up to additive terms. Such additive guarantees are easily ruled out for our problem, even for a static expert. As is typically the case in online algorithms, our performance guarantees are in the form of (multiplicative) competitive ratios rather than (additive) regret bounds.
\end{enumerate}

\paragraph{Our Contributions.}
Our first contribution is to formalize the online covering problem with multiple predictions (\Cref{sec:ocp}). Recall that in each online step, along with a new constraint, the algorithm receives $k$ feasible suggestions for satisfying the constraint. Using these suggestions, we design an algorithm for obtaining a fractional solution to the online covering problem--that we call the \ocp algorithm (\Cref{sec:ocp-algorithm}). To compare this algorithm to the best suggested solution, we define a benchmark \dynamic that captures the minimum cost (fractional) solution that is consistent with at least one of the suggestions in each online step.
\begin{itemize}
    \item [--] Our main technical result shows that the cost of the solution produced by the \ocp algorithm is at most $O(\log k)$ times that of the \dynamic solution.
\end{itemize}    
It is noteworthy that unlike in the classical online covering problem (without predictions), the competitive ratio is independent of the problem size, and only depends on the number of suggestions $k$. As the number of suggestions increases, the competitive ratio degrades because the suggestions have higher entropy (i.e., are less specific). As two extreme examples, consider $k=1$, in which case it is trivial for an algorithm to exactly match the \dynamic benchmark simply by following the suggestion in each step. In contrast, when $k$ is very large, the set of suggested solutions can essentially include all possible solutions, and therefore, the suggestions are useless.

The analysis of the \ocp algorithm makes careful use of potential functions that might be of independent interest. But, while the analysis of the \ocp algorithm is somewhat intricate, we note that the algorithm itself is extremely simple.
\begin{enumerate}
    \item[--] We show that the competitive ratio of $O(\log k)$ obtained by the \ocp algorithm is tight. We give a lower bound of $\Omega(\log k)$ by only using binary (0/1) coefficients and unit cost for each variable, which implies that the lower bound holds even for the special case of the unweighted set cover problem.
    \item[--] Using standard techniques, we observe that the \ocp algorithm can be {\em robustified}, i.e., along with being $O(\log k)$-competitive against the best suggested solution, the algorithm can be made $O(\alpha)$-competitive against the optimal solution where $\alpha$ is the competitive ratio of the best online algorithm (without predictions).
\end{enumerate}
We then use the \ocp algorithm to solve two classic problems--online {\bf set cover} (\Cref{sec:setcover}) and {\bf caching} (\Cref{sec:caching})--in the multiple predictions setting. 
\begin{itemize}
    \item[--] We generalize the online covering framework by introducing {\em box}-type constraints (\Cref{sec:box}). We show that our techniques and results from online covering extend to this more general setting. 
\end{itemize}
We then use this more general formulation for solving the classical online {\bf facility location} problem (\Cref{sec:facility}).

\section{The Online Covering Framework}\label{sec:ocp}%In this section, we formalize our problem and state our results. 
\subsection{Problem Statement}
We define the {\em online covering problem} (\ocp) as follows. 
%In the integral version, 
There are $n$ non-negative variables $\{x_i: i\in [n]\}$ where each $x_i\in [0, 1]$. Initially, $x_i = 0$ for all $i\in [n]$. A linear objective function $c(x) := \sum_{i=1}^n c_i x_i$ is also given offline. In each online step, a new {\em covering} constraint is presented, the $j$-th constraint being given by $\sum_i a_{ij} x_i \ge 1$ where $a_{ij} \ge 0$ for all $i\in [n]$.\footnote{A more general definition allows constraints of the form $\sum_{i=1}^n a_{ij} x_i \ge b_j$ for any $b_j > 0$, but restricting $b_j$ to $1$ is without loss of generality since we can divide throughout by $b_j$ without changing the constraint.} The algorithm is only allowed to {\em increase} the values of the variables, and has to satisfy the new constraint when it is presented. (We denote the total number of constraints by $m$.) %(Note that $\sum_{i=1} a_{ij} \ge 1$ since the constraint cannot be satisfied otherwise.) 
The goal is to minimize the objective function $c(x)$.  
We write this succinctly below:
%\[
%    \min_{x_i\in \{0, 1\}: i\in [n]}  \left\{\sum_{i=1}^n c_i x_i: \sum_i a_{ij} x_i \ge 1 ~\forall j\in [m]\right\}. 
%\]
%In the {\em fractional} version, each variable $x_i$ can take any value in $[0, 1]$ instead of only $0$ or $1$. I.e.:
\[
    \min_{x_i\in [0, 1]: i\in [n]}  \left\{\sum_{i=1}^n c_i x_i: \sum_{i=1}^{n} a_{ij} x_i \ge 1 ~\forall j\in [m]\right\}. 
\]

This framework captures a large class of algorithmic problems such as (fractional) set cover, caching, etc. that have been extensively studied in the online algorithms literature. Our goal will be to obtain a generic algorithm for \ocp with multiple suggestions. When the $j$-th constraint is presented online, the algorithm also receives $k$ suggestions of how the constraint can be satisfied. We denote the $s$-th suggestion for the $j$-th constraint by variables $x_i(j, s)$; they satisfy $\sum_{i=1}^n a_{ij} x_i(j, s) \ge 1$, i.e., all suggestions are feasible. %In recent literature, the case of a single suggestion has been extensively studied for many online covering problems under the umbrella of learning-augmented online algorithms. The main contribution of this work is to understand the case of multiple suggestions.

To formally define the {\em best} suggestion, we say that a solution $\{x_i: i\in [n]\}$ is {\em supported} by the suggestions $\{x_i(j): i\in [n], j\in [m]\}$ if $x_i\ge x_i(j)$ for all $j\in [m]$. Using this definition, we consider below two natural notions of the best suggestion that we respectively call the {\em experts setting} and the {\em multiple predictions setting}.

\paragraph{The Experts Setting.} In the experts setting, there are $k$ experts, and the $s$-th suggestion for each constraint comes from the same fixed expert $s\in [k]$ (say some fixed \ml algorithm or a human expert). The online algorithm is required to be competitive with the {\em best} among these $k$ experts\footnote{This is similar to the experts model in online learning, hence the name.}. To formalize this, we define the benchmark:
\[
   \static =  \min_{s\in [k]} \sum_{i=1}^n c_i \cdot \max_{j\in [m]} x_i(j, s).
\]
%and the corresponding competitive ratio is: 
%\[
%   \comp_{\static} = \frac{\algo}{\static},
%\]
%where $\algo = \sum_{i=1}^n c_i x_i$ is the cost of the algorithm's solution after all $m$ constraints have been satisfied. 
Note that $\{\max_{j\in [m]} x_i(j, s): i\in[n]\}$ is the {\em minimal} solution that is supported by the suggestions of expert $s$; hence, we define the cost of the solution proposed by expert $s$ to be the cost of this solution.

\paragraph{The Multiple Predictions Setting.} In the multiple predictions setting, we view the set of $k$ suggestions in each step as a bag of $k$ predictions (without indexing them specifically to individual predictors or experts) and the goal is to obtain a solution that can be benchmarked against the best of these suggestions in each step. Formally, our benchmark is the minimum-cost solution that is supported by at least one suggestion in each online step:
%
%We say a solution $\hat{x}_i$ for $i\in [n]$ is {\em supported} by the suggestions if for every constraint $j\in [m]$, at least one of the $k$ suggestions is dominated the solution, i.e., $\forall j\in [m], \exists s\in [k]: x_i(j, s) \le \hat{x}_i$. We denote the set of all solutions supported by the suggestions $S = \{x_i(j, s): j\in [m], s\in [k]\}$ by $X(S)$. In the multiple predictions setting, we benchmark the algorithm against the best supported solution:
\[
   \dynamic = \min_{\mathbf{\hat{x}}\in \hat{X}} \sum_{i=1}^n c_i \cdot \hat{x}_i, \text{ where}
\]
\[
   \hat{X} = \{\mathbf{\hat{x}}: \forall i\in [n], \forall j\in [m], \exists s\in [k], \hat{x}_i\ge x_i(j, s)\}.
\]
%and the corresponding competitive ratio is: 
%\[
%   \comp_{\dynamic} = \frac{\algo}{\dynamic}.
%\]
Note that every solution that is supported in the experts setting is also supported in the multiple predictions setting.
%In fact, an equivalent definition of \static is the same as that of \dynamic, except that the existential quantifier over $s$ appears before the universal quantifiers over $i, j$, i.e.,
%\[
%   \static = \min_{\hat{x}\in \hat{X}} \sum_{i=1}^n c_i \cdot \hat{x}_i, \text{ where } \hat{X} = \{\hat{x}: \exists s, \forall i, \forall j,  \hat{x}_i\ge x_i(j, s)\}.
%\]
This implies that $\static \ge \dynamic$, and therefore, %$\rho_{\static} \le \rho_{\dynamic}$ for a given algorithm. As a consequence, 
the competitive ratios that we obtain in the multiple predictions setting also hold in the experts setting. Conversely, the lower bounds on the competitive ratio that we obtain in the experts setting also hold in the multiple predictions setting.

%\alert{Need to add an example here demonstrating illustrating the two benchmarks and highlighting their difference. We should also draw a parallel to static and dynamic experts.}

\subsection{Our Results}

We obtain an algorithm for \ocp with the following guarantee in the multiple predictions setting (and therefore also in the experts setting by the discussion above):
\begin{theorem}\label{thm:ocp-upper}
    There is an algorithm for the online covering problem with $k$ suggestions that has a competitive ratio of $O(\log k)$, even against the  \dynamic benchmark.
\end{theorem}
Note that this competitive ratio is independent of the size of the problem instance, and only depends on the number of suggestions. In contrast, in the classical online setting, the competitive ratio (necessarily) depends on the size of the problem instance.

Next, we show that the competitive ratio in \Cref{thm:ocp-upper} is tight by showing a matching lower bound. This lower bound holds %even for the special case of the online fractional set cover problem (which is the special case of \ocp where all coefficients $a_{ij} \in \{0, 1\}$ and all costs $c_i = 1$), and 
even in the experts setting (hence, by the discussion above, it automatically extends to the multiple predictions setting):
\begin{theorem}\label{thm:ocp-lower}
    The competitive ratio of any algorithm for the online covering problem with $k$ suggestions is $\Omega(\log k)$, even against the \static benchmark.
\end{theorem}
\eat{Perhaps we should switch the order of the theorem statements, since currently the proof of \cref{thm:ocp-robust-gen} is before \Cref{thm:ocp-lower}}
We noted earlier that it is desirable for online algorithms to have {\em robustness} guarantees, i.e., that the algorithm does not fare much worse than the best online algorithm (without predictions) even if the predictions are completely inaccurate. Our next result is the robust version of \Cref{thm:ocp-upper}:

\begin{theorem}
    \label{thm:ocp-robust-gen}
        Suppose a class of online covering problems have an online algorithm (without predictions) whose competitive ratio is $\alpha$. Then, 
        there is an algorithm for this class of online covering problems with $k$ suggestions that produces an online solution whose cost is at most $O(\min\{\log k \cdot \text{\dynamic}, \alpha \cdot \opt\})$.
\end{theorem}
We will prove \Cref{thm:ocp-upper} in the next section. The proofs of \Cref{thm:ocp-lower} and \Cref{thm:ocp-robust-gen} are given in \Cref{sec:lower} and \Cref{sec:robust} respectively. Subsequently, we apply the algorithmic framework developed in \Cref{thm:ocp-upper} to obtain tight competitive ratios for specific instantiations of \ocp, namely the set cover problem (\Cref{sec:setcover}) and the caching problem (\Cref{sec:caching}). Finally, we extend our \ocp result to include box-type constraints (\Cref{sec:box}) and apply it to the online facility location problem (\Cref{sec:facility}).

\section{Online Covering Algorithm}\label{sec:ocp-algorithm}%\subsection{Algorithm for Online Covering Problem}

Recall that in the online covering problem, the new constraint that arrives in the $j$-th online step is $\sum_{i=1}^n a_{ij} x_i \ge 1$ and the algorithm receives $k$ suggestions where the $s$-th suggestion is denoted $x_i(j, s)$. If the current solution of the algorithm given by the variables $x_i$ is feasible, i.e., $\sum_{i=1}^n a_{ij} x_i \ge 1$, then the algorithm does not need to do anything. Otherwise, the algorithm needs to increase these variables until they satisfy the constraint. Next, we describe the rules governing the increase of variables.

Intuitively, the rate of the increase of a variable $x_i$ should depend on three things. First, it should depend on the cost of this variable in the objective, namely the value of $c_i$; the higher the cost, the slower we should increase this variable. Second, it should depend on the contribution of variable $x_i$ in satisfying the new constraint, namely the value of $a_{ij}$; the higher this coefficient, the faster should we increase the variable. Finally, the third factor is how \textit{strongly} $x_i$ has been suggested. To encode this mathematically, we first make the assumption that every suggestion is {\em tight}, i.e., 
\begin{equation}\label{eq:tight}
    \sum_{i=1}^n a_{ij} x_i(j, s) = 1 \text{ for every suggestion } s\in [k].
\end{equation}
This assumption is without loss of generality because, if not, we can decrease the variables $x_i(j, s)$ in an arbitrary manner until the constraint becomes tight. (Note that this change can only decrease the cost of the benchmark solutions $\dynamic$ and $\static$; hence, any competitive ratio bounds obtained after this transformation also hold for the original set of suggestions.) 

Having made all the suggestions tight, we now encode how strongly a variable has been suggested by using its {\em average} suggested value $\frac{1}{k}\cdot \sum_{s=1}^k x_i(j, s)$. Our algorithm (see \Cref{alg:ocp}) increases all variables $x_i$ satisfying $x_i < \frac12$ simultaneously at rates governed by these parameters; precisely, we use 
\[
    \frac{dx_i}{dt}  = \frac{a_{ij}}{c_i} \left(x_i + \delta\cdot x_{ij}\right), 
    \text{ where } \delta = \frac1k, x_{ij} = \sum_{s=1}^k x_i(j, s).
\]    
The algorithm continues to increase the variables until $\sum_{i=1}^n a_{ij} x_i \ge \frac12$; along the way, any variable $x_i$ that reaches $\frac12$ is dropped from the set of increasing variables. To satisfy the $j$-th constraint, we note that the variables $2 x_i$ are feasible for the constraint. (Note that since all variables $x_i \le \frac12$ before the scaling, every variable can be doubled without violating $x_i\le 1$.) Since this last step of multiplying every variable by $2$ only increases the cost of the algorithm by a factor of $2$, we ignore this last scaling step in the rest of the analysis.
\begin{algorithm}[h]
\caption{Online Covering Algorithm}
\textbf{Offline:} All variables $x_i$ are initialized to $0$.\\
\textbf{Online:} On arrival of the $j$-th constraint:\\
\hspace*{10pt} \textbf{while} $\sum_{i=1}^n a_{ij} x_i < \frac12$, \\
\hspace*{20pt}    \textbf{for} $i\in [n]$\\
\hspace*{30pt}\textbf{if} $x_i < \frac12$,
increase $x_i$ by 
$\frac{dx_i}{dt}  = \frac{a_{ij}}{c_i} \left(x_i + \delta\cdot x_{ij}\right)$,\\
\hspace*{70pt}where $\delta = \frac1k$ and $x_{ij} = \sum_{s=1}^k x_i(j, s)$.
\label{alg:ocp}
\end{algorithm}

Before analyzing the algorithm, we note that although we described it using a continuous process driven by a differential equation, the algorithm can be easily discretized and made to run in polynomial time where in each discrete step, some variable $x_i$ reaches $\frac12$ (and therefore, $x_i$ cannot increase any further) or $\sum_{i=1}^n a_{ij} x_i$ reaches $\frac12$ (and therefore, the algorithm ends for the current online step). In this section, we will analyze the continuous algorithm rather than the equivalent discrete algorithm for notational simplicity.

Next, we show that the algorithm is valid, i.e., that there is always a variable $x_i$ that can be increased inside the {\bf while} loop. If not, then we have $\sum_{i=1}^n a_{ij} x_i < \frac12$ but $x_i\ge \frac12$ for all variables $x_i$, $i\in [n]$. This implies that $\sum_{i=1}^n a_{ij} < 1$, which is a contradiction because the constraint $\sum_{i=1}^n a_{ij} x_i \ge 1$ is then unsatisfiable by any setting of variables $x_i \le 1$. (In particular, this would mean that there cannot be any feasible suggestion for this constraint.)

Now, we are ready to bound the competitive ratio of \Cref{alg:ocp} with respect to the \dynamic benchmark.
First, we bound the rate of increase of algorithm's cost:

\begin{lemma}\label{lem:alg-cost}
    The rate of increase of cost in \Cref{alg:ocp} is at most $\frac32$.
\end{lemma}
\begin{proof}
The rate of increase of cost is given by:
\begin{align*}
\sum_{i=1}^n c_i\cdot \frac{dx_i}{dt}
&= \sum_{i=1}^n a_{ij} \left(x_i + \delta\cdot x_{ij}\right)\\
&= \sum_{i=1}^n a_{ij} x_i + \frac{1}{k}\cdot \sum_{i=1}^n \sum_{s=1}^k a_{ij} x_i(j,s)\\ 
&< \frac12 + \frac{1}{k}\cdot \sum_{s=1}^k \left(\sum_{i=1}^n  a_{ij} x_i(j,s)\right) =\frac32,
\end{align*}
where we used $\sum_{i=1}^n a_{ij} x_i < \frac12$ from the condition on the {\bf while} loop, and $\sum_{i=1}^n a_{ij} x_i(j,s) = 1$ for all $s\in [k]$ from \Cref{eq:tight}.
\end{proof}

We now define a carefully crafted non-negative potential function $\phi$. We will show that the potential decreases at constant rate when \Cref{alg:ocp} increases the variables $x_i$ (Lemma~\ref{lem:alg-pot}). By \Cref{lem:alg-cost}, this implies that the potential can pay for the cost of \Cref{alg:ocp} up to a constant. We will also show that the potential $\phi$ is at most $O(\log k)$ times the \dynamic benchmark (Lemma~\ref{lem:opt-pot}). Combined, these yield \Cref{thm:ocp-upper}.

Let $\xd_i$ denote the value of variable $x_i$ in the $\dynamic$ benchmark. The potential function for a variable $x_i$ is then defined as follows:
\[
    \phi_i = c_i \cdot \xd_i \cdot \ln \frac{(1+\delta) \xd_i}{x_i + \delta \xd_i}, \text{ where } \delta = \frac{1}{k}.
\]
and the overall potential is:
\[
    \phi = \sum_{i: \xd_i\ge x_i} \phi_i.
\]
The intuition behind only including those variables that have $\xd_i \ge x_i$ in the potential function is that the potential stores the excess cost paid by the \dynamic benchmark for these variables so that it can be used later to pay for increase in the algorithm's variables.

First, we verify that the potential function is always non-negative.

\begin{lemma}
\label{lem:pot-nonneg}
    For any values $x_i, \xd_i$ of the variables, the potential function $\phi$ is non-negative.
\end{lemma}
\begin{proof}
    Note that $\phi$ only includes variables $x_i$ such that $\xd_i \ge x_i$. For such variables, 
    \[
        \phi_i 
        = c_i \cdot \xd_i \cdot \ln \frac{(1+\delta) \xd_i}{x_i + \delta \xd_i} 
        = c_i \cdot \xd_i \cdot \ln \frac{1+\delta}{\frac{x_i}{\xd_i} + \delta} 
        \ge 0.\qedhere
    \]
\end{proof}

Next, we bound the potential as a function of the variables $\xd_i$ in the \dynamic benchmark:
\begin{lemma}
\label{lem:opt-pot}
    The potential $\phi_i$ for variable $x_i$ is at most $c_i \xd_i \cdot \ln \left(1+\frac{1}{\delta}\right) = c_i \xd_i \cdot O(\log k)$. As a consequence, the overall potential $\phi \le O(\log k)\cdot \sum_{i=1}^n c_i \xd_i$.
\end{lemma}
%Note that the cost of the \dynamic benchmark is $\sum_{i=1}^n c_i \xd_i$; hence, this lemma establishes that $\phi$ is at most $O(\log k)$ times the cost of the benchmark solution.

\begin{proof}
We have 
\begin{align*}
    \phi_i &= c_i \xd_i \cdot \ln \frac{(1+\delta) \xd_i}{x_i + \delta \xd_i}\\
    &\le c_i \xd_i \cdot \ln \frac{(1+\delta) \xd_i}{\delta \xd_i}\\
    &= c_i \xd_i \cdot \ln \left(1+\frac{1}{\delta}\right)\\
    &= c_i \xd_i \cdot O(\log k).\qedhere
\end{align*}

\end{proof}

Finally, we bound the rate of decrease of potential $\phi$ with increase in the variables $x_i$ in \Cref{alg:ocp}. Our goal is to show that up to constant factors, the decrease in potential $\phi$ can pay for the increase in cost of the solution of \Cref{alg:ocp}.
\begin{lemma}
\label{lem:alg-pot} 
    The rate of decrease of the potential $\phi$ with increase in the variables $x_i$ in \Cref{alg:ocp} is at least $\frac12$.
\end{lemma}
\begin{proof}
Recall that $\phi_i = c_i\cdot \xd_i\cdot \ln \frac{(1+\delta)\xd_i}{x_i+\delta \xd_i}$. Therefore,
\begin{equation*}%\label{eq:phix}
    \frac{d \phi_i}{d x_i} = - c_i\cdot \frac{\xd_i}{x_i+\delta \xd_i}.
\end{equation*}
Recall that in \Cref{alg:ocp}, the rate of increase of variables $x_i$ is given by $\frac{dx_i}{dt}  = \frac{a_{ij}}{c_i} \left(x_i + \delta\cdot x_{ij}\right)$, where $\delta = \frac1k$ and $x_{ij} = \sum_{s=1}^k x_i(j, s)$.
Thus, we have:
\begin{align*}
    \frac{d \phi_i}{dt} 
    = \frac{d \phi_i}{d x_i} \cdot \frac{d x_i}{dt}
    &= - c_i\cdot \frac{\xd_i}{x_i+\delta \xd_i} \cdot \frac{a_{ij}}{c_i}(x_i + \delta\cdot x_{ij})\\
    &= - a_{ij}\cdot \xd_i\cdot \frac{x_i + \delta x_{ij}}{x_i + \delta \xd_i}.
\end{align*}
Now, we have two cases:
\begin{itemize}
    \item If $x_{ij} \ge \xd_i$, then 
    \begin{align}\label{eq:case1}
        &\frac{d \phi_i}{dt} 
        = - a_{ij}\cdot \xd_i\cdot \frac{x_i + \delta x_{ij}}{x_i + \delta \xd_i} \nonumber\\
        \le& - a_{ij}\cdot \xd_i\cdot \frac{x_i + \delta \xd_i}{x_i + \delta \xd_i}
        = - a_{ij}\xd_i.
    \end{align}
    \item If $x_{ij} < \xd_i$, then 
    \begin{align}\label{eq:case2}
        &\frac{d \phi_i}{dt} 
        = - a_{ij}\cdot \xd_i\cdot \frac{x_i + \delta\cdot x_{ij}}{x_i + \delta \xd_i} \nonumber\\
        =& - a_{ij}\cdot x_{ij}\cdot \frac{\frac{\xd_i}{x_{ij}}\cdot x_i + \delta\cdot \xd_i}{x_i + \delta \xd_i}
        < - a_{ij} x_{ij}.
    \end{align}
\end{itemize}
We know that at least one of the suggestions in the $j$-th step is supported by the \dynamic benchmark. Let $s(j)\in [k]$ be such a supported suggestion. Then,
\begin{align*}
    x_{ij} &= \sum_{s=1}^k x_i(j, s) \ge x_i(j, s(j)), \text{ and }\\
    \xd_i &\ge x_i(j, s(j)) \text{ since } s(j) \text{ is supported by \dynamic}.
\end{align*}
Therefore, in both cases (\Cref{eq:case1} and \Cref{eq:case2}) above, we get
\[
    \frac{d\phi_i}{dt} \le - a_{ij} x_i(j, s(j)).
\]
Let us denote $I_j = \{i \mid x_i(j, s(j ))\ge x_i\}$. 
Then, the total decrease in potential is given by:
\begin{equation}\label{eq:final}
    \frac{d\phi}{dt} 
    = \sum_{i:\xd_i\ge x_i} \frac{d\phi_i}{dt} 
    \le - \sum_{i\in I_j} a_{ij} x_i(j, s(j)).
\end{equation}
By feasibility of the $s(j)$-th suggestion for the $j$-th constraint, we have $\sum_{i=1}^n a_{ij} x_i(j, s(j)) \ge 1$. Therefore,
\begin{align*}
    \sum_{i\in I_j} a_{ij}x_i(j, s(j)) + \sum_{i\notin I_j} a_{ij}x_i(j, s(j)) &\ge 1\\
    \text{i.e., } \sum_{i\in I_j} a_{ij}x_i(j, s(j)) + \sum_{i\notin I_j} a_{ij}x_i &> 1,
\end{align*}    
since $x_i> x_i(j,j(s))$ for $i\notin I_j$. Thus,
\begin{align*}
    \sum_{i\in I_j} a_{ij}x_i(j, s(j)) + \sum_i a_{ij}x_i &> 1\\ 
%    \text{(since } a_{ij} x_i\ge 0 \text{ for all } i, j)&\\
    \text{i.e., } \sum_{i\in I_j} a_{ij}x_i(j, s(j)) &> \frac12,
\end{align*}    
since $\sum_{i=1}^n a_{ij}x_i < \frac12$ in \Cref{alg:ocp}.
The lemma follows by \Cref{eq:final}.\qedhere
\end{proof}

\Cref{thm:ocp-upper} now follows from the above lemmas using 
standard arguments as follows:
\begin{proof}[Proof of \Cref{thm:ocp-upper}]
    Initially, let $x_i = 0$ for all $i\in [n]$ but let $\xd_i$ be their final value. 
    %(Note that the algorithm doesn't set the values of $\xd_i$, in fact it doesn't even know the values -- this is just for bounding the cost in the analysis.) 
    Then, by \Cref{lem:opt-pot}, the potential $\phi$ is at most $O(\log k)$ times the cost of \dynamic. Now, as \Cref{alg:ocp} increases the values of the variables $x_i$, it incurs cost at rate at most $\frac32$ (by \Cref{lem:alg-cost}) and the potential $\phi$ decreases at rate at least $\frac12$ (by \Cref{lem:alg-pot}). Since $\phi$ is always non-negative (by \Cref{lem:pot-nonneg}), it follows that the total cost of the algorithm is at most $3$ times the potential $\phi$ at the beginning, i.e., at most $O(\log k)$ times the \dynamic benchmark. This completes the proof of \Cref{thm:ocp-upper}.
\end{proof}

%end eat
\subsection{Robust Algorithm for the Online Covering Problem}\label{sec:robust}

Now, we prove the robust version of \Cref{thm:ocp-upper}, namely \Cref{thm:ocp-robust-gen}.
\begin{proof}[Proof of \Cref{thm:ocp-robust-gen}]
    We run a meta algorithm with two sets of suggestions corresponding to two solutions. The first solution is obtained by using \Cref{alg:ocp} with $k$ suggestions. By~\Cref{thm:ocp-upper} this solution has cost at most $O(\log k) \cdot \dynamic$. The second solution is produced by the online algorithm that achieves a competitive ratio of $\alpha$ in the statement of \Cref{thm:ocp-robust-gen}. Using \Cref{alg:ocp} again for the meta algorithm, \Cref{thm:ocp-robust-gen} now follows by invoking \Cref{thm:ocp-upper}.
\end{proof}

\eat{
To obtain the latter theorem, we first apply \Cref{alg:ocp} with the $k$ suggestions to obtain an online solution whose cost is at most $O(\log k) \cdot \dynamic$ by \Cref{thm:ocp-upper}. 
A second solution is produced by the online algorithm that achieves a competitive ratio of $\alpha$ in the statement of \Cref{thm:ocp-robust-gen}. We then use a meta algorithm with two suggestions corresponding to these two solutions. Using \Cref{alg:ocp} again in the meta algorithm, \Cref{thm:ocp-robust-gen} now follows from \Cref{thm:ocp-upper} for the meta algorithm.
}
As one particular application of \Cref{thm:ocp-robust-gen}, we note that for general \ocp, the best competitive ratio is due to the following result of Gupta and Nagarajan~\cite{GuptaN14} (see also Buchbinder and Naor~\cite{BuchbinderN09a}):
\begin{theorem}[Gupta and Nagarajan~\cite{GuptaN14}]
    \label{thm:ocp-online}
    There is an algorithm for the online covering problem that has a competitive ratio of $O(\log d)$, where $d$ is the maximum number of variables with non-zero coefficients in any constraint.
\end{theorem}
 This automatically implies the following corollary of \Cref{thm:ocp-robust-gen}:
\begin{theorem}
    \label{thm:ocp-robust}
        There is an algorithm for the fractional online covering problem that produces an online solution whose cost is at most $O(\min\{\log k \cdot \text{\dynamic}, \ln d\cdot \opt\})$ in the multiple predictions setting with $k$ predictions, where $d$ is the maximum number of non-zero coefficients in any constraints of the online covering problem instance.
\end{theorem}
For specific problems that can be modeled as \ocp, it might be possible to obtain a better competitive ratio than $O(\log d)$ by using the structure of those instances. In that case, the competitive ratio in the multiple predictions setting also improves accordingly by \Cref{thm:ocp-robust-gen}.

\subsection{Lower Bound for the Online Covering Problem}\label{sec:lower}

Here we show that the competitive ratio obtained in \Cref{thm:ocp-upper} is tight, i.e., we prove \Cref{thm:ocp-lower}. We will restrict ourselves to instances of \ocp where $a_{ij} \in \{0, 1\}$ and $c_i = 1$ for all $i, j$; this is called the {\em online (fractional) set cover} problem. (We will discuss the set cover problem in more detail in the next section.) Moreover, our lower bound will hold even in the experts model, i.e., against the \static benchmark. Since the \dynamic benchmark is always at most the \static benchmark, it follows that the lower bound also holds for the \dynamic benchmark. \eat{Since we are in the experts model, we index the experts $\{1, 2, \ldots, k\}$.}

\begin{proof}[Proof of \Cref{thm:ocp-lower}]

We index the $k$ experts $\{1, 2\ldots k\}$ using a uniform random permutation. We will construct an instance of \ocp, where the cost of the optimal solution is $T$. The instance has $k$ rounds, where in each round there are $T$ constraints. We index the $j$th constraint of the $i$th round as $(i, j)$ for $i\in [k], j\in [T]$. There are $kT$ variables that are also indexed as $(i, j)$ for $i\in [k], j\in [T]$. Constraint $(i, j)$ is satisfied by each of the variables $(i', j)$ for all $i'\ge i$ (i.e., $a_{(i,j), (i',j)} = 1$). When constraint $(i, j)$ is presented (in round $i$), expert $i'$ for every $i'\ge i$ sets variable $(i', j)$ to $1$ to satisfy it. (The suggestions of experts $i''< i$ are immaterial in this round, and they can set any arbitrary variable satisfying constraint $(i, j)$ to $1$.) 

The optimal solution is to follow expert $k$, i.e., the variables $(k, j)$ for all $j\in [T]$ should be set to $1$; this has cost $T$. After round $i$, the cumulative expected cost of any deterministic algorithm across the variables $(i, 1), (i, 2), \ldots, (i, T)$ is at least $\frac{T}{i}$. Across all $i\in[k]$, this adds up to a total expected cost
%\begin{equation*}
    $T\cdot\left(1+\frac12+\ldots+\frac{1}{k}\right) = \Omega(T\log k)$.
%\end{equation*} 
The theorem then follows by Yao's minimax principle \cite{yao1977probabilistic}. 
\end{proof}

%end eat

\section{Online Set Cover}\label{sec:setcover}In the (weighted) set cover problem, we are given a collection of subsets $\cal S$ of a ground set $U$, where set $S\in {\cal S}$ has weight $w_S$. The goal is to select a collection of sets ${\cal T}\subseteq {\cal S}$ of minimum cost $\sum_{T\in {\cal T}} w_T$ that cover all the elements in $U$, i.e., that satisfies $\cup_{T\in {\cal T}} T = U$. In the {\em online} version of this problem (see \cite{AlonAABN09}), the set of elements in $U$ is not known in advance. In each online step, a new element $u\in U$ is revealed, and the sets in $\cal S$ that contain $u$ are identified. If $u$ is not covered by the sets in the current solution $\cal T$, then the algorithm must {\em augment} $\cal T$ by adding some set containing $u$ to it. 

In the fractional version of the set cover problem, sets can be selected to fractions in $[0, 1]$, i.e., a solution is given by variables $x_S \in [0, 1]$ for all $S\in {\cal S}$. The constraint is that the total fraction of all sets containing each element $u$ must be at least $1$, i.e., $\sum_{S: u\in S} x_S \ge 1$ for every element $u\in U$. The cost of the solution is given by $\sum_{S\in {\cal S}} w_S x_S$. Clearly, this is a special case of the online covering problem in the previous section where each variable $x_i$ represents a set and each constraint $\sum_{i=1}^n a_{ij} x_i \ge 1$ is for an element, where $a_{ij} = 1$ if an only if element $j$ is in set $i$, else $a_{ij} = 0$. 

We define the {\em frequency} of an element to be the number of sets containing it, and denote the maximum frequency of any element by $d$. Note that this coincides with the maximum number of non-zero coefficients in any constraint. The following theorem is an immediate corollary of \Cref{thm:ocp-robust}:

\begin{theorem}
    \label{thm:fracset}
    There is an algorithm for the fractional online set cover problem that produces an online solution whose cost is at most $O(\min\{\ln k \cdot \text{\dynamic}, \ln d\cdot \opt\})$ in the multiple predictions setting with $k$ predictions.
\end{theorem}

It is interesting to note that when the suggestions are good in the sense that $\dynamic = O(\opt)$, the competitive ratio of $O(\log k)$ in the above theorem is independent of the size of the problem instance. In contrast, for the classical fractional online set cover problem, there is a well-known lower bound of $\Omega(\log d)$ on the competitive ratio. We also note that the competitive ratio in \Cref{thm:fracset} is tight since as we noted earlier, the lower bound instance constructed in \Cref{thm:ocp-lower} is actually an instance of the set cover problem. 

\subsection{Online Rounding and the Integral Set Cover Problem}
Next, we consider the integral set cover problem, i.e., where the variables $x_i$ are required to be integral, i.e., in $\{0, 1\}$ rather than in $[0, 1]$. The following a well-known result on rounding fractional set cover solutions online:
\begin{theorem}[Alon~{\em et al.}~\cite{AlonAABN09}]
    Given any feasible solution to the fractional online set cover problem, there is an online algorithm for finding a feasible solution to the integral online set cover problem whose cost is at most $O(\log m)$ times that of the fractional solution, where $m$ is the number of elements.
\end{theorem}

By applying this theorem on the fractional solution produced by \Cref{thm:fracset}, we get a competitive ratio of $O(\log m\log k)$ for the integral online set cover problem with $k$ predictions against the \dynamic benchmark:
\begin{theorem}
    \label{thm:intset}
    There is an algorithm for the integral online set cover problem that produces an online solution whose cost is at most $O(\min\{\ln m \ln k \cdot \text{\dynamic}, \ln m\ln d\cdot \opt\})$ in the multiple predictions setting with $k$ predictions.
\end{theorem}
It is well-known that the competitive ratio of the integral online set cover problem (without predictions) is at least $\tilde{\Omega}(\log m\log d)$~\cite{AlonAABN09}\footnote{The notation $\tilde{\Omega}$ (and $\tilde{O}$) hide lower order terms.}. In the degenerate case of $k = d$, any instance of online set cover can be generated in the $k$ predictions settings where the benchmark solution \dynamic will be the optimal solution, since all the sets containing an element can be provided as suggestions in each online step. As a consequence, the competitive ratio in \Cref{thm:intset} is tight (up to lower order terms).

\eat{
Below, we show that this bound cannot be unconditionally improved:

\begin{theorem}
Given $k$ experts, the competitive ratio for the integral online set cover problem is at least $O(\log k\cdot \log n)$ for certain values of $k$.
\end{theorem}

\begin{proof}
Consider the case where $k=f$ where $f$ is the maximum number of sets that an element belongs to. In this case, the suggestions will include all the sets that an element can belong to (and therefore, include no extra information vis-a-vis the general online case). We know that there exists a lower bound construction showing that the competitive ratio is greater than $(\log f\cdot \log n)$ for a general online case.
\end{proof}

}%end eat
\section{(Weighted) Caching}\label{sec:caching}The caching problem is among the most well-studied online problems (see, e.g., \cite{SleatorT85,FiatKLMSY91,McGeochS91,AchlioptasCN00,young1991online,BlumBK99}, and the textbook~\cite{BorodinE98}).
In this problem, there is a set of $n$ pages and a cache that can hold any $h$ pages at a time.\footnote{The usual notation for cache size is $k$, but we have changed it to $h$ since we are using $k$ to denote the number of suggestions.} In every online step $j$, a page $p_j$ is requested; if this page is not in the cache, then it has to be brought into the cache (called {\em fetching}) by evicting an existing page from the cache. In the weighted version (see, e.g., \cite{chrobak1991new,Young94,BansalBN07}), the cost of fetching a page $p$ into the cache is given by its non-negative weight $w_p$. (In the unweighted version, $w_p = 1$ for all pages.) The goal of a caching algorithm is to minimize the total cost of fetching pages while serving all requests.

The (weighted) caching problem can be formulated as online covering problem by defining variables $x_p(r) \in \{0, 1\}$ to indicate whether page $p$ is evicted between its $r$-th and $(r+1)$-st requests. Let 
%$j(p,r) = 1$ denote that the $r$-th request for page $p$ is in the $j$-th online step, and let 
$r(p,j)$ denote the number of times page $p$ is requested until (and including) the $j$-th request. For any online step $j$, let $B(j) = \{p: r(p,j) \ge 1\}$ denote the set of pages that have been requested until (and including) the $j$-th request. The covering program formulation is as follows:
\begin{equation*}
\begin{array}{ll@{}ll}
\text{min}  & \sum_p \sum_r w_p\cdot x_p(r) \text{ subject to}&\\
 &\sum_{p \in B(j), p \neq p_j} x_p(r(p,j)) \geq |B(j)|-h,   & ~\forall j \ge 1\\
& x_p(r) \in \{0,1\}, & \forall p, \forall r \ge 1\\
\end{array}
\end{equation*}
In the fractional version of the problem, we replace the constraints $x_p(r)\in \{0, 1\}$ with the constraints $x_p(r)\in [0, 1]$. Clearly, this fits the definition of the online covering problem.\footnote{Strictly speaking, we need to scale the first set of coefficients by $|B(j)|-h$, but as we mentioned earlier, this is equivalent since scaling the coefficients has no bearing on the competitive ratio of our algorithm.} Moreover, for the fractional weighted caching problem, Bansal, Buchbinder, and Naor gave an online algorithm with a competitive ratio of $O(\log h)$:
\begin{theorem}[\cite{BansalBN07}]
    \label{thm:caching-frac}
    There is an online algorithm for the fractional weighted caching problem with a competitive ratio of $O(\log h)$.
\end{theorem}
Note that the competitive ratio of $O(\log h)$ is better than that given by \Cref{thm:ocp-online} since the cache size $h$ is typically much smaller than the total number of pages. Now, we apply \Cref{thm:ocp-robust-gen} to get the following result for fractional weighted caching with $k$ predictions:

\begin{theorem}
    \label{thm:fraccache}
    There is an algorithm for the fractional weighted caching problem that produces an online solution whose cost is at most $O(\min\{\ln k \cdot \text{\dynamic}, \ln h\cdot \opt\})$ in the multiple predictions setting with $k$ predictions.
\end{theorem}
\subsection{Integral (Weighted) Caching}
Next, we consider the integral weighted caching problem, i.e., where the variables $x_p(r)$ have to be in $\{0, 1\}$ rather than $[0, 1]$. We use the following known result about online rounding of fractional weighted caching solutions:

\begin{theorem}[Bansal, Buchbinder, and Naor~\cite{BansalBN07}]
    \label{thm:caching-rounding}
    Given any feasible online solution to the fractional weighted caching problem, there is an online algorithm for finding a feasible online solution to the integral weighted caching problem whose cost is at most $O(1)$ times that of the fractional solution.
\end{theorem}

By applying this theorem on the fractional solution produced by \Cref{thm:fraccache}, we get a competitive ratio of $O(\log k)$ for the integral weighted caching problem with $k$ predictions against the \dynamic benchmark:

\begin{theorem}
    \label{thm:intcache}
    There is an online algorithm for the integral weighted caching problem that produces an online solution whose cost is at most $O(\min\{\ln k \cdot \text{\dynamic}, \ln h\cdot \opt\})$ in the multiple predictions setting with $k$ predictions.
\end{theorem}

\section{Online Covering with Box Constraints}\label{sec:box}In this section, we generalize the online covering framework in \Cref{sec:ocp} by allowing additional {\em box} constraints of the form $x_{ij} \le y_i$. The new linear program is given by:
\begin{align}
\text{Minimize } \sum_{i=1}^n c_i y_i & + \sum_{i=1}^{n}\sum_{j=1}^{m} d_{ij} x_{ij} \text{ subject to} \nonumber\\
   x_{ij} &\leq y_{i} \quad \forall i, j \label{eq:fac1}\\
   \sum_{i=1}^{n} a_{ij} x_{ij} &= 1 \quad \forall j \label{eq:fac2} \\
    y_i &\in [0, 1] \quad \forall i
\end{align}
The box constraints $x_{ij} \le y_i$ do not have coefficients and hence are known offline. As in \ocp, the cost coefficients $c_i$ are known offline. In each online step, a new covering constraint $\sum_{i=1}^n a_{ij} x_{ij} \ge 1$ is revealed to the algorithm, and the corresponding cost coefficient $d_{ij}$ is also revealed.

We first note that this is a generalization of the online covering problem. To see this, set $d_{ij} = 0$. Then, we get:
\[
    \min_{y_i\in [0, 1]: i\in [n]}  \left\{\sum_{i=1}^n c_i y_i: \sum_{i=1}^n a_{ij} y_i \ge 1 ~\forall j\in [m]\right\},
\]
which is precisely the online covering problem.

The more general version captures problems like facility location (shown in the next section) that are not directly modeled by \ocp. We denote the $s$-th suggestion for the $j$-th constraint by variables $y_i(j, s)$ and $x_{ij}(s)$; they satisfy $\sum_{i=1}^n a_{ij} x_{ij}(s) \ge 1$, i.e., all suggestions are feasible.

\subsection{Online Algorithm}

Our algorithm for \ocp with box constraints is given in~\Cref{alg: Dynamic_OPT_fac_loc}. As earlier, 
%we raise a time variable $t$ uniformly and 
for all $i$, we simultaneously raise $x_{ij}$ (and possibly $y_i$ as well) at the rate specified by the algorithm. As in the case of \ocp, we raise these variables only till the constraint is satisfied up to a factor of $\frac12$, and any individual variable does not cross $\frac12$. This allows us to double all variable and satisfy the constraints at an additional factor of $2$ in the cost. Moreover, the same argument as in \Cref{alg:ocp} implies that this algorithm is also feasible, i.e., there is at least one variable that can be increased in the {\bf while} loop.

\begin{algorithm}[H]
\caption{Algorithm for Online \fac}
\label{alg: Dynamic_OPT_fac_loc}
On arrival of a new constraint $\sum_{i=1}^n a_{ij} x_{ij} = 1$: \\
\hspace*{20pt} \textbf{Initialize} $x_{ij} = 0, \quad \forall j$. \\
\hspace*{20pt} Set $\Gamma_{ij} := \sum_{s=1}^{k} x_{ij}(s), \quad \forall j$ and $\delta = \frac1k$. \\
\hspace*{50pt} \textbf{while $\sum_j x_{ij} < \frac{1}{2}$}  \\
\hspace*{80pt} \textbf{for all} $i$ such that $x_{ij} < \frac12$: \\
\hspace*{100pt} \textbf{if} $x_{ij} < y_i$ \\
\hspace*{120pt} Increase  $x_{ij}$ at the rate $\frac{\partial x_{ij}}{\partial t} = \left(\frac{a_{ij}}{d_{ij}}\right)\cdot\left(x_{ij} + \delta\cdot \Gamma_{ij}\right)$  \\
\hspace*{100pt} \textbf{else} \\
\hspace*{120pt} Increase both variables $x_{ij}, y_i$ at the same rate \\
\hspace*{130pt} $\frac{\partial y_i}{\partial t} = \frac{\partial x_{ij}}{\partial t} = \left(\frac{a_{ij}}{d_{ij}+c_i}\right)\cdot\left(x_{ij} + \delta\cdot \Gamma_{ij}\right) $.\\
\end{algorithm}

\subsection{Analysis}
In this section, we analyze~\Cref{alg: Dynamic_OPT_fac_loc}. 
%Fix an index $i$, and consider the input till step $i$. Let $x^\opt, y^\opt$ be the best dynamic solution till this step.  
We first bound the rate of increase of the cost of the algorithm. 
\begin{lemma}
\label{lem:costfac}
The rate of increase of the cost for~\Cref{alg: Dynamic_OPT_fac_loc} is at most $\frac32$.
\end{lemma}
\begin{proof}
Consider~\Cref{alg: Dynamic_OPT_fac_loc} when the constraint $\sum_{i=1}^n a_{ij} x_{ij} =1$ arrives. 
For any $i$, we claim:
\begin{align}
    \label{eq:proof1}
    d_{ij}\cdot \frac{\partial x_{ij}}{\partial t}+ c_i\cdot \frac{\partial y_i}{\partial t} =
    a_{ij} \left(x_{ij} + \delta\cdot \Gamma_{ij}\right). 
\end{align}
To show this, there are two cases to consider:
\begin{itemize}
    \item $x_{ij} < y_i$: In this case, $\frac{\partial y_i}{\partial t} = 0$, and  $d_{ij}\cdot \frac{\partial x_{ij}}{\partial t} 
    = a_{ij} \left(x_{ij} + \delta\cdot \Gamma_{ij}\right)$ and so~\eqref{eq:proof1} holds. 
    \item $x_{ij} = y_i$: In this case, $\frac{\partial y_i}{\partial t} 
    = \frac{\partial x_{ij}}{\partial t} 
    = \frac{a_{ij}}{d_{ij}+c_i} \cdot \left(x_{ij} + \frac{\Gamma_{ij}}{k}\right)$, i.e.,
    $c_i\cdot \frac{\partial y_i}{\partial t} + d_{ij}\cdot \frac{\partial x_{ij}}{\partial t} 
    = a_{ij} \left(x_{ij} + \delta\cdot \Gamma_{ij}\right)$,
    and so~\eqref{eq:proof1} holds. 
\end{itemize}
Therefore, the rate of change of the objective function is given by:
\[
    \sum_{i}\left(d_{ij}\cdot \frac{\partial x_{ij}}{\partial t}+ c_i\cdot \frac{\partial y_i }{\partial t}\right) 
    \stackrel{\eqref{eq:proof1}}{=} \sum_{i} a_{ij} \left(x_{ij} + \delta\cdot \Gamma_{ij}\right) 
    = \sum_{i}a_{ij} x_{ij} + \frac{\sum_{s} \sum_{i} a_{ij} x_{ij}(s)}{k}
    \le \frac12 + 1
    = \frac32.\qedhere
\]
\end{proof}

We now describe the potential function, which has a similar structure as that in the online covering framework. 
Let $\xd_{ij}, \yd_i$ denote the values of the variables $x_{ij}, y_i$ respectively in the benchmark solution $\dynamic$. The potential function for a variable $x_{ij}$ is then defined as follows:
\[
    \phi_{ij} = d_{ij} \cdot \xd_{ij} \cdot \ln \frac{(1+\delta) \xd_{ij}}{x_{ij} + \delta \xd_{ij}}, \text{ where } \delta = \frac1k,
\]
and the potential for the variable $y_i$ is given by 
$$  \psi_i = c_i \cdot \yd_i \cdot \ln \frac{(1+\delta) \yd_i}{y_i + \delta \yd_i}.$$
As before, the overall potential is 
\[
    \phi = \sum_{i,j: \xd_{ij}\ge x_{ij}} \phi_{ij} +\sum_{i: \yd_i\ge y_i} \psi_i .
\]
The rest of the proof proceeds along the same lines as that for the online covering problem. But we give the details for the sake of completeness.

The next lemma is the analogue of~\Cref{lem:pot-nonneg}, and shows that the potential $\phi$ is always non-negative. 
\begin{lemma}
\label{lem:pot-nonneg1}
    For any values of the variables $x_{ij},y_i,\xd_{ij}, \yd_i$, the potential function $\phi$ is non-negative.
\end{lemma}
\begin{proof}
We show that each of the quantities $\phi_{ij}, \psi_i$ in the expression for $\phi$ is non-negative. Consider a pair $i,j$ for which $\xd_{ij} \geq x_{ij}$. Then
       $$ \phi_{ij} = d_{ij} \cdot \xd_{ij} \cdot \ln \frac{(1+\delta) \xd_{ij}}{x_{ij} + \delta \xd_{ij}} 
        = d_{ij} \cdot \xd_{ij} \cdot \ln \frac{1+\delta}{\frac{x_{ij}}{\xd_{ij}} + \delta} \ge 0.
    $$
    Similarly, $\psi_i \geq 0$ if $\yd_i \geq y_i$. This shows that $\phi \geq 0$. 
    \qedhere
\end{proof}

Now we bound the potential against the benchmark solution \dynamic. The proof of this lemma is very similar to that of~\Cref{lem:opt-pot}. 

\begin{lemma}
\label{lem:opt-pot1}
    The following bounds hold: $\phi_{ij} \le d_{ij}\cdot \xd_{ij} \cdot \ln (1+\frac{1}{\delta}) = d_{ij} \xd_{ij} \cdot O(\log k)$ and $\psi_i \le c_i \cdot \yd_i \cdot \ln (1+\frac{1}{\delta})) = c_i \yd_i \cdot O(\log k)$. As a consequence, the overall potential $\phi \le O(\log k)\cdot \left(\sum_i \sum_j d_{ij} \xd_{ij} + \sum_j c_i \yd_i\right)$.
\end{lemma}
\begin{proof}
We have 
\[
    \phi_{ij} 
    = d_{ij} \xd_{ij} \cdot \ln \frac{(1+\delta) \xd_{ij}}{x_i + \delta \xd_{ij}}
    \le d_{ij} \xd_{ij} \cdot \ln \frac{(1+\delta) \xd_{ij}}{\delta \xd_{ij}}
    = d_{ij} \xd_{ij} \cdot \ln \left(1+\frac{1}{\delta}\right)
    = d_{ij} \xd_{ij} \cdot O(\log k).
\]    
The bound for $\psi_i$ also follows similarly.
\end{proof}

Finally, we bound the rate of decrease of potential $\phi$ with increase in the variables in \Cref{alg: Dynamic_OPT_fac_loc}. 
\begin{lemma}
\label{lem:alg-pot1} 
    The rate of decrease of the potential $\phi$ in \Cref{alg: Dynamic_OPT_fac_loc} is at least $\frac12$.
\end{lemma}
\begin{proof}
It is easy to check that 
\begin{equation}\label{eq:phix1}
    \frac{\partial \phi_{ij}}{\partial x_{ij}} = - \frac{d_{ij} \xd_{ij}}{x_{ij}+\delta \xd_{ij}}, \quad \frac{\partial \psi_i}{\partial y_i} = - \frac{c_i \yd_i}{y_i+\delta \yd_i}.
\end{equation}

Consider the step when the $j$-th constraint arrives. We claim that for any index $i \in [n]$, 
\begin{align}
    \label{eq:phix2}
    \frac{\partial \left( \phi_{ij} + \psi_i \right)}{\partial t} \leq - a_{ij} \xd_{ij} \cdot \frac{x_{ij} + \delta \Gamma_{ij}}{x_{ij} + \delta \xd_{ij}}.
\end{align}

To prove this, we consider two cases: 
\begin{itemize}
    \item $x_{ij} < y_i$: In this case, $\frac{\partial x_{ij}}{\partial t} = \frac{a_{ij}}{d_{ij}} \cdot \left( x_{ij} + \delta \Gamma_{ij} \right), \frac{\partial y_i}{\partial t} = 0$. Combining this with~\eqref{eq:phix1}, we see that 
    $$ \frac{\partial \left( \phi_{ij} + \psi_i \right)}{\partial t} = \frac{\partial \phi_{ij}}{\partial t} = \frac{\partial \phi_{ij}}{\partial x_{ij}} \cdot \frac{\partial x_{ij}}{\partial t} 
    = -a_{ij} \xd_{ij} \cdot \frac{x_{ij} + \delta \Gamma_{ij}}{x_{ij} + \delta \xd_{ij}}.$$
    \item $x_{ij} = y_i$: In this case, $\frac{\partial x_{ij}}{\partial t} = \frac{\partial y_i}{\partial t} = \frac{a_{ij}}{d_{ij}+c_i} \cdot \left( x_{ij} + \frac{\Gamma_{ij}}{k} \right). $
    Using~\eqref{eq:phix1} again and the fact that $y_i = x_{ij}$, we see that 
    $$ \frac{\partial \left( \phi_{ij} + \psi_i \right)}{\partial t} = \frac{\partial \phi_{ij}}{\partial x_{ij}} \cdot \frac{\partial x_{ij}}{\partial t} + \frac{\partial \psi_i}{\partial y_i} \cdot \frac{\partial y_i}{\partial t} 
    = - \frac{a_{ij}}{c_i + d_{ij}} \left( \frac{d_{ij} \xd_{ij}}{x_{ij} + \delta \xd_{ij}} + \frac{c_i \yd_i}{x_{ij} + \delta \yd_i}\right)  \cdot  \left( x_{ij} + \delta \Gamma_{ij} \right).$$
    Since $\yd_i \geq \xd_{ij}, \frac{\yd_i}{x_{ij} + \delta \yd_i} \geq \frac{\xd_{ij}}{x_{ij} + \delta \xd_{ij}}.$
Therefore, the RHS above is at most $ - a_{ij} \xd_{ij} \cdot \frac{x_{ij} + \delta \Gamma_{ij}}{x_{ij} + \delta \xd_{ij}}.$    
\end{itemize}

Thus, we have shown that inequality~\eqref{eq:phix2} always holds. 
Now, we have two cases:
\begin{itemize}
    \item  $\Gamma_{ij} \ge \xd_{ij}$: Inequality~\eqref{eq:phix2} implies that 
    \[
        \frac{\partial (\phi_{ij}+ \psi_i)}{\partial t} \le -a_{ij} \xd_{ij}.
    \]
    \item $\Gamma_{ij} < \xd_{ij}$: Using~\eqref{eq:phix2} again, we see that 
    \[
        \frac{d (\phi_{ij}+\psi_i)}{dt} 
        \leq - a_{ij} \xd_{ij} \cdot \frac{x_{ij} + \delta\cdot \Gamma_{ij}}{x_{ij} + \delta \xd_{ij}}
        = - a_{ij} \Gamma_{ij}\cdot \frac{\frac{\xd_{ij}}{\Gamma_{ij}}\cdot x_{ij} + \delta\cdot \xd_{ij}}{x_{ij} + \delta \xd_{ij}}
        \leq - a_{ij} \Gamma_{ij}.
    \]
\end{itemize}
We know that at least one of the suggestions in the $i$-th step is supported by the \dynamic benchmark. Let $s(j) \in [k]$ be the index of such a supported suggestion. Then,
\begin{align*}
    \Gamma_{ij} &= \sum_{s=1}^k x_{ij}(s) \geq x_{ij}(s(j)), \text{ and }\\
    \xd_{ij} &\ge x_{ij} (s(j)) \text{ since } s(j) \text{ is supported by \dynamic}.
\end{align*}
Therefore, in both cases above, we get
\[
    \frac{\partial (\phi_{ij}+ \psi_i)}{\partial t} \le - a_{ij}\cdot x_{ij}(s(j)).
\]
%
%\noindent
Let $I$ denote the index set $\{i\in [n]: \xd_{ij} \geq x_{ij}\}$ (here $j$ is fixed). 
We claim that the 
total decrease in potential satisfies:
\begin{align}
\label{eq:decpot}
    \frac{\partial \phi}{\partial t} \leq -\sum_{i \in I} a_{ij}\cdot x_{ij}(s(j)).
\end{align}
To see this, let $I'$ denote the index set $\{i: \yd_i \geq y_i \}$. Then the term in $\phi$ corresponding to step $j$ is $\phi_j := \sum_{i \in I} \phi_{ij} + \sum_{i \in I'} \psi_i. $ First consider an index $i \in I$ for which $x_{ij} < y_i$. In this case, we know that $\frac{\partial \psi_i}{\partial t} = 0$, and so irrespective of whether $i$ belongs to $I'$ or not, the rate of change of the terms in $\phi_j$ corresponding to $i$ is 
$$ \frac{\partial (\phi_{ij} + \psi_i)}{\partial t} \leq -a_{ij}\cdot x_{ij} (s(j)). $$ 
Now consider an index $i \in I$ for which $x_{ij} = y_i$. Since $\yd_i \geq \xd_{ij}$, it follows that $\yd_i \geq y_i$ and so $i \in I'$ as well. Therefore, the rate of change of the terms in $\phi_j$ corresponding to $i$ is equal to 
$$\frac{\partial (\phi_{ij} + \psi_i)}{\partial t} \leq -a_{ij}\cdot x_{ij} (s(j)). $$
Finally, consider an index $i \in I' \setminus I$. It is easy to verify that $\frac{\partial \psi_i}{\partial t} \leq 0$. Thus, inequality~\eqref{eq:decpot} follows.

The rest of the argument proceeds as in the proof of~\Cref{lem:alg-pot}.
By feasibility of the $s(j)$-th suggestion in step $j$, we have:
\begin{align*}
    \sum_j  a_{ij} x_{ij}(s(j)) &\ge 1\\
    \text{i.e., } \sum_{i \in I} a_{ij} x_{ij} (s(j)) + \sum_{i \notin I} a_{ij} x_{ij}(s(j)) &\ge 1\\
    \text{i.e., } \sum_{i \in I} a_{ij} x_{ij} (s(j)) + \sum_{i \notin I} a_{ij} x_{ij} &> 1 \quad \quad \left(\text{since } i \notin I, \text{ we have } x_{ij} > \xd_{ij} \geq x_{ij}(s(j))\right)\\
     \text{i.e., } \sum_{i \in I} a_{ij} x_{ij}(s(j)) + \sum_{i=1}^n a_{ij} x_{ij} &> 1 \\
     \text{i.e., } \sum_{i \in I} a_{ij} x_{ij}(s(j))  &> \frac12 \quad \left(\text{since } \sum_{i=1}^n x_{ij} < \frac12 \text{ in \Cref{alg: Dynamic_OPT_fac_loc}}\right).
\end{align*}
The desired result now follows from~\eqref{eq:decpot}. 
\end{proof}

We now combine these lemmas to show the following result:
\begin{theorem}
    \label{thm:box}
    There is an algorithm for the online covering problem with box constraints that produces an online solution whose cost is at most $O(\log k) \cdot \dynamic$ in the multiple predictions setting with $k$ predictions.
\end{theorem}
\begin{proof}
    Initially, let $x_{ij} = y_i = 0$ for all $i, j$ but let $\xd_{ij}, \yd_i$ be their final value. 
    %(Note that the algorithm doesn't set the values of $\xd_i$, in fact it doesn't even know the values -- this is just for bounding the cost in the analysis.) 
    Then, by \Cref{lem:opt-pot1}, the potential $\phi$ is at most $O(\log k)$ times the cost of \dynamic. Now, as \Cref{alg: Dynamic_OPT_fac_loc} increases the values of the variables $x_{ij}, y_i$, it incurs cost at rate at most $\frac32$ (by \Cref{lem:costfac}) and the potential $\phi$ decreases at rate at least $\frac12$ (by \Cref{lem:alg-pot1}). Since $\phi$ is always non-negative (by \Cref{lem:pot-nonneg1}), it follows that the total cost of the algorithm is at most $3$ times the potential $\phi$ at the beginning, i.e., at most $O(\log k)$ times the \dynamic benchmark. 
\end{proof}

We can also robustify the solution produced by \Cref{alg: Dynamic_OPT_fac_loc} using the same ideas as in \Cref{thm:ocp-robust-gen}. Namely, we run \Cref{alg: Dynamic_OPT_fac_loc} to produce one solution and an online algoeithm without prediction to obtain another solution. Then, these two solutions are fed into a meta algorithm running \Cref{alg: Dynamic_OPT_fac_loc} to obtain the final solution. We get the following analog of \Cref{thm:ocp-robust-gen}:

\begin{theorem}
    \label{thm:box-robust-gen}
        Suppose a class of online covering problems with box constraints have an online algorithm (without predictions) whose competitive ratio is $\alpha$. Then, 
        there is an algorithm for this class of online covering problems with box constraint with $k$ suggestions that produces an online solution whose cost is at most $O(\min\{\ln k \cdot \text{\dynamic}, \alpha \cdot \opt\})$.
\end{theorem}

\section{Online Facility Location}\label{sec:facility}%\subsection{Application to Online \fac}
We now apply the online covering with box constraints framework to the online \fac problem. We first describe the offline version of this problem. An instance of the \fac is given by a metric space $d(\cdot, \cdot)$,  a set of $n$ potential facility location points $\cal F$, and a set $\cal R$ of clients in the  metric space. 
Further, each location $f_i \in {\cal F}$ has a facility opening cost $o_i$. 
A solution opens a subset ${\cal F}' \subseteq {\cal F}$ of facilities, and assigns each client $r_j$ to the closest open facility $f_{i_j} \in {\cal F}'$. The connection cost of this client is $d(r_j, f_{i_j})$ and the goal is to minimize the sum of the connection costs of all the clients and the facility opening costs of the facilities in ${\cal F}'.$ 

In the online \fac problem, clients arrive over time and the solution needs to assign each arriving client to an open facility. The algorithm can also open facilities at any time. This problem was first studied by Meyerson~\cite{Meyerson01} who gave a competitive ratio of $O(\log n)$ for $n$ clients. This was later improved (and derandomized) by Fotakis~\cite{Fotakis08} to $O(\frac{\log n}{\log\log n})$. Since then, many variants of the problem have been studied; for a survey of the results, the reader is referred to~\cite{Fotakis11}.

We now describe the natural LP relaxation for \fac; our notion of a fractional solution will be with respect to this LP relaxation. 
We have variables $y_i$ for each facility location $f_i \in {\cal F
}$, which denotes the extent to which facility $f_i$ is open, and variables $x_{ij}$ for each client $r_j$ and facility $f_i$ denoting the extent to which $r_j$ is assigned to $f_i$. The relaxation is as follows:

\begin{align*}
\text{Minimize } \sum_{i=1}^m o_i y_i & + \sum_{i=1}^{n}\sum_{j=1}^{m} d_{ij} x_{ij} \text{ subject to} \nonumber\\
   x_{ij} &\leq y_{i} \quad \forall i, j\\ %\label{eq:fac1}\\
   \sum_{i=1}^{n}x_{ij} &= 1 \quad \forall i\\ %\label{eq:fac2} \\
    y_i &\in [0, 1] \quad \forall j
\end{align*}

Notice that this is a special case of the online covering problem with box constraints, where $c_i=o_i$ for all $i$, and $ a_{ij}=1$ for all $i,j$.  Thus, we can apply~\Cref{alg: Dynamic_OPT_fac_loc} for online \fac problem as well. Moreover, as mentioned earlier, there is an algorithm (due to Fotakis~\cite{Fotakis08}) for the online facility location problem (without predictions) that has a competitive ratio of at most $O(\log n)$ for $n$ clients. As a corollary of~\Cref{thm:box-robust-gen}, we get the following result:

\begin{theorem}
    \label{thm:fracfl}
    There is an algorithm for the fractional online \fac problem that produces an online solution whose cost is at most $O(\min\left\{\frac{\ln n}{\ln \ln n}\cdot \opt, \ln k \cdot \dynamic\right\})$ in the multiple predictions setting with $k$ predictions.
\end{theorem}

We also note that the $O(\log k)$ bound in \Cref{thm:fracfl} cannot be improved further (except lower order terms). This is because we can simulate an instance of the online facility location problem without predictions by giving all prior client locations as suggested facility locations for $k=n$. Furthermore, a lower bound of $\Omega\left(\frac{\log n}{\log \log n}\right)$ is also known (due to Fotakis~\cite{Fotakis08}) for online facility location, and this lower bound construction easily extends to the fractional version of the problem.

\section{Future Directions}\label{sec:future}This paper presented a general recipe for the design of online algorithms with multiple machine-learned predictions. This general technique was employed to obtain tight results for classical online problems such as set cover, (weighted) caching, and facility location. We believe this framework can be applied to other online covering problems as well. In particular, it would be interesting to consider the $k$-server problem with multiple suggestions in each step specifying the server that should serve the new request. It would also be interesting to extend this framework to some packing problems such as online budgeted allocation, and to more general settings for mixed packing and covering linear programs and non-linear (convex) objectives. From a technical standpoint, our potential-based analysis introduces several new ideas that can be useful even in the competitive analysis of classical online algorithms.

{\small
\bibliographystyle{alpha}
\bibliography{ref}
}

\end{document}